\documentclass[twoside,leqno,twocolumn]{article}  
\usepackage{ltexpprt}

\usepackage{hyperref}
\usepackage{eurosym}
\usepackage{color}
\usepackage{graphicx}
\graphicspath{{./figures/}}

\usepackage[font=small,labelfont=bf]{caption}
\usepackage{subfig}
\usepackage{float}
\usepackage{algcompatible}

\usepackage{amsfonts}
\usepackage{amssymb}
\usepackage{amsmath}
\usepackage{mathrsfs}

\usepackage{algorithm}
\usepackage{algpseudocode}

\begin{document}

\title{\Large Uplift Modeling with Multiple Treatments and General Response Types}

\author{Yan Zhao\thanks{Department of Electrical Engineering and Computer Science, Massachusetts Institute of Technology. zhaoyanmit@gmail.com}  
\and Xiao Fang\thanks{Department of Electrical Engineering and Computer Science, Massachusetts Institute of Technology. ustcfx@gmail.com} 
\and David Simchi-Levi\thanks{Institute for Data, Systems, and Society, Department of Civil and Environmental Engineering, Operations Research Center,  Massachusetts Institute of Technology. dslevi@mit.edu}}

\date{}

\maketitle


\begin{abstract} \small\baselineskip=9pt 
Randomized experiments have been used to assist decision-making in many areas. They help people select the optimal treatment for the test population with certain statistical guarantee. However, subjects can show significant heterogeneity in response to treatments. The problem of customizing treatment assignment based on subject characteristics is known as uplift modeling, differential response analysis, or personalized treatment learning in literature. A key feature for uplift modeling is that the data is unlabeled. It is impossible to know whether the chosen treatment is optimal for an individual subject because response under alternative treatments is unobserved. This presents a challenge to both the training and the evaluation of uplift models. In this paper we  describe how to obtain an unbiased estimate of the key performance metric of an uplift model, the expected response. We present a new uplift algorithm which creates a forest of randomized trees. The trees are built with a splitting criterion designed to directly optimize their uplift performance based on the proposed evaluation method. Both the evaluation method and the algorithm apply to arbitrary number of treatments and general response types. Experimental results on  synthetic data and industry-provided data show that our algorithm leads to significant performance improvement over other applicable methods. \\

\noindent
\textbf{Accepted:} \emph{2017 SIAM International Conference on Data Mining (SDM2017)}
\end{abstract}

\section{Introduction}\label{sec1}

We often face the situation where we need to identify from a set of alternatives the candidate that leads to the most desirable outcome. For example, doctors want to know which treatment plan is the most effective for a certain disease. In an email marketing campaign, a company needs to select the message that yields the highest click through rate. Randomized experiments are frequently conducted to answer these questions. In such an experiment, subjects are randomly assigned to a treatment and their responses are recorded. Then by some statistical criteria, one treatment is selected as the best. While randomized experiments (also known as A/B testing in online settings) have been helpful in many areas, it has one major shortcoming - disregard for subject heterogeneity. A medical treatment that is the most effective over the entire patient population may be ineffective or even detrimental for patients with certain conditions. An email message that leads to the highest click-though-rate overall might tend to put off customers in some subpopulation. Therefore, it is of great interest to develop models that can correctly predict the optimal treatment based on given subject characteristics. This has been referred to as the personalized treatment selection problem or \textbf{Uplift Modeling} in the literature.

While appearing similar to classification problem at the first sight, uplift modeling poses unique challenges. In a randomized experiment, it is impossible to know whether the chosen treatment is optimal for any \emph{individual} subject because response under alternative treatments is unobserved. As a result, data collected from a randomized experiment is \emph{unlabeled} in the classification perspective because the true value of the quantity that we are trying to predict (the optimal treatment) is unknown even on the training data. 

Perhaps the most obvious approach to Uplift Modeling is what we call the Separate Model Approach (SMA). We first create one predictive model for each treatment. Given a new data, we can obtain the predicted response under each treatment with the corresponding model, and then select the treatment with the best predicted response. The main advantage of this approach is that it does not require development of new algorithms and software. Any conventional classification/regression method can be employed to serve the purpose. Applications of SMA include  direct marketing \cite{Hansotia2002} and customer retention \cite{Manahan2005}. However, the Separate Model Approach, while simple and correct in principle, does not always perform well in real-world situation \cite{Lo2002}\cite{Radcliffe2011}. One reason for this is the mismatch between the objective used for training the models and the actual purpose of the models. When the uplift (difference in response between treatments) follows a distinct pattern from the response, SMA will focus on predicting the response rather than the \emph{uplift signal}. See \cite{Radcliffe2011} for an illustrative example. The situation is exacerbated when data is noisy and insufficient or when the uplift is much weaker than the response. Unfortunately, these are usually the cases for practical uplift applications. 

Seeing the shortcomings of the Separate Model Approach, researchers have proposed a number of algorithms that aim at directly modeling the uplift effect. However, almost all of them are designed for the situation with a single treatment. A logistic regression formulation is proposed which explicitly includes interaction terms between features and the treatment \cite{Lo2002}.  Support Vector Machine is adapted for uplift modeling to predict whether a subject will be positively, neutrally, or negatively affected by the treatment \cite{Zaniewicz2013}. The adaption of K-Nearest Neighbors for uplift modeling is briefly mentioned in both \cite{Alemi2009} and \cite{Su2012}. A new subject is simply assigned to the empirically best treatment as measured on the K training data that are closest to it. Several tree-based algorithms have been proposed for uplift modeling, each with a different splitting criterion  \cite{Chickering2000} \cite{Hansotia2002} \cite{Rzepakowski2010} \cite{Radcliffe2011} \cite{Guelman2014}. In \cite{Chickering2000}, the authors modify the standard decision tree construction procedure \cite{CART} by forcing a split on the treatment at each leaf node. In \cite{Hansotia2002} a splitting criterion is employed that maximizes the difference between the difference between the treatment and control probabilities in the left and right child nodes. In \cite{Rzepakowski2010},  splitting points are chosen that maximize the distributional difference between two child nodes as measured by a weighted Kullback-Leibler divergence and a weighted squared Euclidean distance. In \cite{Radcliffe2011}, a linear model is fitted to each candidate split and the significance of the interaction term is used as the measure of the split quality. In \cite{Guelman2014}, the variable that has the smallest p-value in the hypothesis test on the interaction between the response and itself is selected as the splitting variable. Then the splitting point is chosen to maximize the interaction effect. It is demonstrated experimentally that the use of Bagging or Random Forest on uplift trees  often results in significant improvement in performance \cite{Rzepakowski2015}. 

Despite its wide application, literature on the more general multi-treatment uplift problem is limited. Rare exceptions include \cite{Rzepakowski2012} and \cite{clark}. In  \cite{Rzepakowski2012},  the tree-based algorithm described in \cite{Rzepakowski2010} is extended to multiple treatment cases by using a weighted sum of pairwise distributional divergence as the splitting criterion. In \cite{clark}, a multinomial logit formulation is proposed in which treatments are incorporated as binary features. They also explicitly include the interaction terms between treatments and features.  Moreover, finite sample convergence guarantees are established for model parameters and out-of-sample performance guarantee. Both methods can handle binary as well as discrete response type. It is worth mentioning that  the causal K-nearest neighbors originally intended for single treatment can be naturally generalized to multiple treatments \cite{Alemi2009}  \cite{Su2012}. This algorithm is implemented in a R package called \emph{uplift} by Leo Guelman \cite{Ruplift}.

One of challenges facing uplift modeling is how to accurately evaluate model performance off-line using the randomized experiment data. For single treatment cases, \emph{qini curves} and \emph{uplift curves} have been used to serve the purpose \cite{Radcliffe2007} \cite{Rzepakowski2010}. The problem with them as performance metrics is that them do not measure the key quantity of interest - the expected response. What they measure is some surrogate quantity which is hopefully close to the increase in expected response relative to only applying the control. In \cite{Rzepakowski2010}, the authors explained that they use uplift curves because there does not seem to be a better option at the time. 

We now describe the contribution of our paper. In Section~\ref{sec2}, we discuss how to obtain an unbiased estimate of the expected response under an uplift model. The method applies to arbitrary number of treatments and general response types (binary, discrete, continuous). It works even when the treatments and the control are not evenly distributed which is often the case in practice. The method also allows us to compute the confidence interval of an estimate of the expected response. Furthermore, we introduce the \emph{modified uplift curve} which plots the expected response as a function of the percentage of population subject to treatments. As we discuss more in section~\ref{sec4-2}, the modified uplift curve provides a fair way to compare uplift models. 

Based on the new evaluation method, we propose a tree-construction procedure with a splitting criterion that explicitly optimizes the performance of the tree as measured on the training data. This idea is in line with the machine learning philosophy of loss minimization on the training set. We use an ensemble of trees to mitigate the overfitting problem that commonly happens with a single tree. We refer to our algorithm as the CTS algorithm where the name stands for Contextual Treatment Selection. The performance of CTS is tested on three benchmark data sets. The first is a  50-dimensional synthetic data set. The latter two are randomized experiment data provided by our industry collaborators. On all of the data sets, CTS demonstrates superior performance compared to other applicable methods which include Separate Model Approach with Random Forest/Support Vector Regression/K-Nearest Neighbors/AdaBoost, and Uplift Random Forest (upliftRF) as implemented in the R \emph{uplift} package \cite{Ruplift}.

The remainder of the paper is organized as follows. In Section~\ref{sec2}, we first introduce the formulation of the multi-treatment uplift modeling and then present the unbiased estimate of the expected response for an uplift model. Section~\ref{sec3} describes the CTS algorithm in detail. In Section~\ref{sec4} we present the setup and the results of the experimental evaluation. The modified uplift curve is also introduced in this Section. Section~\ref{sec5} ends the paper with a brief summary and ideas for future research.

\section{Evaluation of Uplift Models}\label{sec2}

Before introducing the evaluation method, we first describe the mathematical formulation of uplift problems and the notation used throughout this paper.

\subsection{Problem Formulation and Notation}\label{sec2-1}

We use upper case letters to denote random variables and lower case letters their realizations. We use boldface for vectors and normal typeface for scalers. 
\begin{itemize}
	\item $\mathbf{X}$ represents the feature vector and $\mathbf{x}$ its realization. Subscripts are used to indicate specific features. For example, $X_j$ is the $j$th feature in the vector and $x_j$ its realization. Let $\mathbb{X}^d$ denote the $d$-dimensional feature space. 
	
	\item $T$ represents the treatment. We assume there are $K$ different treatments encoded as $\{1,\ldots,K \}$. The control group is indicated by $T=0$. 
	
	\item Let $Y$ denote the response and $y$ its realization. Throughout this paper we assume the larger the value of $Y$, the more desirable the outcome. 
\end{itemize}

In the email marketing campaign example mentioned in the Introduction where the company wants to customize the messages to maximize the click through rate, $\mathbf{X}$ would be the charactering information of customers such as the browsing history, the purchase pattern, demographics, etc..  $T$ would be the different versions of the email message. And the response $Y$ would be the 1/0 variable indicating whether a customer click the message or not. 

Suppose we have a data set of size $N$ containing the joint realization of $(\mathbf{X}, T, Y)$ collected from a randomized experiment. We use superscript $(i)$ to index the samples as below.
$$
s_N = \big\{\, \big(\, \mathbf{x}^{(i)}, t^{(i)}, y^{(i)}  \,\big), i=1,\ldots,N \,\big\}.
$$

An uplift model $h$ is a mapping from the feature space to the space of treatments, or $h(\cdot): \mathbb{X}^d \rightarrow \{0, 1,\ldots,K\}$. The key performance metric of an uplift model is the expected value of the response if the model is used to assign the treatment, 
\begin{align}
\mathrm{E}[\, Y|T=h(\mathbf{X}) \,]. \label{eqn:E}
\end{align}

As is with classification and regression problems, the optimal expected response achievable in an uplift problem is determined by the underlying data generation process. The optimal expected response is achieved by a model $h^*$ that satisfies the point-wise optimal condition, i.e., $\forall \mathbf{x}\in \mathbb{X}^d$, 
\begin{equation}
h^*(\mathbf{x})\in \arg\max_{t=0,1,...,K}\; \mathrm{E}[Y|\mathbf{X}=\mathbf{x}, T=t]. \label{eqn:hstar}
\end{equation}

\subsection{Model Evaluation}\label{sec2-2}

One way of looking at an uplift model is that it partitions the entire feature space into disjoint subspaces and assigns each subspace to one treatment. With data from a randomized experiment, it is possible to estimate the probability of a sample falling into any subspace as well as the expected response in that subspace under the assigned treatment. Then by the law of total expectation we can estimate the expected response in the entire feature space.

In a randomized experiment, treatments are assigned randomly and independently from the features. However, treatments are not necessarily evenly distributed. Let $p_t$ denote the probability that a treatment is equal to $t$. In any meaningful situation we will have $p_t>0$ for $t=0,...,K$. 

\begin{lemma} Given an uplift model $h$, define a new random variable 
\begin{equation}
Z = \sum_{t=0}^K \frac{1}{p_t} Y \mathbb{I}\{h(\mathbf{X})=t\}  \mathbb{I}\{T=t\}
\end{equation}	
where $\mathbb{I}\{\cdot\}$ is the $0/1$ indicator function. Then 
$$
\mathrm{E}[Z] = \mathrm{E}[Y|T=h( \mathbf{X} )].
$$
\end{lemma}
\begin{proof} The proof is straightforward using the law of total expectation.
\begin{align*}
& \mathrm{E} \big[Z\big] \\
=\; & \sum_{t=0}^K \; \frac{1}{p_t} \mathrm{E} \big[ Y \mathbb{I}\{\, h\big( \mathbf{X} \big) = t \,\} \big| T=t \big] P\{ T=t \}  \\
=\; & \sum_{t=0}^K \;\mathrm{E} \big[ Y \big| h(\mathbf{X})=t, T=t \big] P\{h(\mathbf{X})=t\} \\
=\; &\; \mathrm{E} \big[Y|h(\mathbf{X})=T \big]
\end{align*}
\end{proof}

Given a set of randomized experiment data $s_N = \big\{\, \big(\, \mathbf{x}^{(i)}, t^{(i)}, y^{(i)}  \,\big), i=1,\ldots,N \,\big\}$, computing the value of $z^{(i)}$ is simple. If for the $i$th sample the predicted treatment matches the actual treatment, then $z^{(i)}$ is equal to $y^{(i)}/p_t$, the actual response scaled by the reverse of the treatment probability. Otherwise, $z^{(i)}$ equals zero. It is well known that the sample average is an unbiased estimate of the expected value. Therefore we have the following theorem.

\begin{theorem}\label{thm1}
The sample average 
\begin{equation}
\bar{z}=\frac{1}{N}\sum_{i=1}^N z^{(i)}  \label{eqn:zbar}
\end{equation}  is an unbiased estimate of $\mathrm{E}[Y|T=h(\mathbf{X} )]$.
\end{theorem}
Moreover, we can compute the confidence interval of $\bar{z}$ which also helps to estimate the possible range of $\mathrm{E}[Y|T=h(\mathbf{X})]$.

\section{The CTS Algorithm}\label{sec3}

Tree-based methods are time-tested tools in Machine Learning \cite{CART}. When combined into ensembles, they prove to be among the most powerful algorithms for general classification and regression problems \cite{jmlr_onehundred}. Even for the relatively new uplift modeling problem, there have been some reports on the excellent performance of tree ensembles  \cite{Rzepakowski2015}. 

The algorithm we present in this section also generates a tree ensemble. We refer to it as the \textbf{CTS} algorithm which stands for Contextual Treatment Selection. What is unique about CTS is its splitting criterion that directly maximizes the expected response from the tree as measured on the training set.

\subsection{Splitting Criterion}\label{sec3-1}

We take the recursive binary splitting approach. Each split creates two new branches further down the tree. Let $\phi$ be the feature space associated with a leaf node. The best we can do for subjects falling into $\phi$ is to assign the subspace-wise optimal treatment. Suppose $s$ is a candidate split that divides $\phi$ into the left child-subspace $\phi_l$ and the right child-subspace $\phi_r$. Because the two child subspaces can have different treatments, the added flexibility leads to increased expected response for subjects in $\phi$ overall. The increase is denoted as $\Delta \mu(s)$ as below.
\begin{align}
\label{eqn:deltamu}
&\Delta\mu(s) \\ 
= \; & P\{\mathbf{X}\in\phi_l | \mathbf{X}\in\phi \}\max_{t_l=0,...,K} \mathrm{E}[Y|\mathbf{X}\in\phi_l, T=t_l]   \nonumber \\
+ \,& P\{\mathbf{X}\in\phi_r | \mathbf{X}\in\phi \}\max_{t_r=0,...,K} \mathrm{E}[Y|\mathbf{X}\in\phi_r, T=t_r]   \nonumber \\
- \,&\max_{t=0,...,K} \mathrm{E}[Y|\mathbf{X}\in\phi, T=t] \nonumber.  
\end{align}

So the idea is straightforward. At each step in the tree-building process, we want to perform the split that brings about the greatest increase in expected response $\Delta\mu$. The question is how to estimate $\Delta\mu$ with training data. Let $\phi'$ stand for one of the child subspace $\phi_l$ or $\phi_r$.  We use $\hat{p}(\phi'|\phi)$ to denote the estimate for the conditional probability of a subject in $\phi'$ given that it is in $\phi$ , and $\hat{y}_t(\phi')$ the estimate for the conditional expected response in subspace $\phi'$ under treatment $t$.

For the conditional probability, we will simply use the sample fraction
\begin{align}
\hat{p}(\phi'|\phi) = &\; \frac{\sum_{i=1}^N \mathbb{I}\{\mathbf{x}^{(i)}\in\phi'\}}{\sum_{i=1}^N \mathbb{I}\{\mathbf{x}^{(i)} \in \phi\}}.
\end{align}

Estimating the conditional expectation requires more effort. First, the estimation is performed by treatments, therefore less data is available.  Second, treatments may not be evenly distributed. It is common to have only a small percentage of population subject to treatments in a randomized experiment. Let $n_t(\phi')$ be the number of samples in $\phi'$ with treatment equal to $t$. With two user-defined parameters $\mathtt{min\_split}$ and $\mathtt{n\_reg}$, the conditional expectation is estimated as follows. \\

\noindent
If $n_t(\phi') < \mathtt{min\_split}$,
\begin{equation}
\hat{y}_t(\phi') = \hat{y}_t(\phi),
\end{equation}
otherwise,
\begin{equation}
\hat{y}_t(\phi')  = \frac{  \sum_{i=1}^N y^{(i)}\mathbb{I}\{ \mathbf{x}^{(i)} \in \phi'\}  \mathbb{I}\{t^{(i)}=t \}  + \hat{y}_t(\phi)\cdot \mathtt{n\_reg}   }{\sum_{i=1}^N \mathbb{I}\{ \mathbf{x}^{(i)} \in \phi'\}  \mathbb{I}\{t^{(i)}=t \}  + \mathtt{n\_reg} } \label{eq:nreg}
\end{equation}

Note that $\hat{y}_t(\phi')$ is defined recursively -  the value of $\hat{y}_t(\phi')$ depends on the corresponding estimate for the parent node $\hat{y}_t(\phi)$. To initialize the definition, the estimated expectation for the root node $\hat{y}_t(\mathbb{X}^d)$ is set to the sample average. We assume there are at least enough samples to estimate expected response accurately in the root node, otherwise customizing treatment selection is impractical. The rational behind the estimation formula is twofold. First, if there are not enough samples for some treatment,  we simply inherit the estimation from the parent node. This mechanism, combined with the termination rules in Section~\ref{sec3-2}, allows the trees to grow to a full extent while ensuring reliable estimate of the expected response. Second, to avoid being misled by outliers, we add a regularity term to the sample average calculation. The larger the value of $\mathtt{n\_reg}$, the more samples it takes to shift the estimate from the parent-estimate $\hat{y}_t(\phi)$ to the actual sample average.  Based on our experiments, it is usually helpful to set $\mathtt{n\_reg}$ to a small positive integer.

To summarize, we estimate the increase in the expected response from a candidate split $s$ as below,
\begin{align}
\hat{\Delta\mu}(s)  = &\quad \hat{p}(\phi_l|\phi)\times \max_{t=0,...,K} \hat{y}_t(\phi_l)   \nonumber\\
& + \hat{p}(\phi_r|\phi) \times \max_{t=0,...,K} \hat{y}_t(\phi_r) -   \max_{t=0,...,K} \hat{y}_t(\phi). \label{eqn:criterion}
\end{align}
At each step of the tree-growing process, the split that leads to the highest estimated increase in expectation is performed.

\subsection{Termination Rules}\label{sec3-2}

Another important component of a tree-based algorithm is the termination rules. In CTS, a node is a terminal node if any of the following conditions is satisfied. The tree growing process terminates when no more splits can be made. 
\begin{enumerate}
\item The number of samples is less than the user-defined parameter $\mathtt{min\_split}$ for all treatments

\item There does not exist a split that leads to non-negative gain in the estimated expected response.

\item All the samples in the node have the same response value.
\end{enumerate}
The first condition allows us to split a node as long as there is at least one treatment containing enough samples. The second condition states that a split should not be executed if it damages the performance of the current tree. We allow a split with zero gain to be performed because a nonprofitable split for the current step may lead to profitable splits in future steps. The third condition is to avoid the split of pure node. Without condition 3), a split will be selected randomly when all samples have the same response value because all possible splits lead to zero gain.

\subsection{The Algorithm}\label{sec3-3}

To mitigate the overfitting problem commonly associated with a single tree, we formulate CTS in a form similar to Random Forest \cite{RandomForest}. A group of trees are constructed based on the splitting criterion and termination rules described previously. Each tree is built on a different bootstrapped training data set. At each step of the learning process, only a random subset of features are considered for splitting. A terminal node of a tree contains the estimation of expected response under each treatment for that node. Given a point in the feature space and a treatment, the predicted expected response from the forest is the average of the predictions from all the trees. The CTS algorithm is outlined in Algorithm~\ref{alg2}.

\begin{algorithm}[t]
	\caption{CTS - Contextual Treatment Selection}
	\begin{algorithmic}[2]
		\item[\bf{Input}:]training data $s_N$, number of samples in each tree \texttt{B} ($\texttt{B}\leq N$), number of trees \texttt{ntree}, number of variables to be considered for a split \texttt{mtry} ($1\leq \texttt{mtry} \leq d$), the minimum number of samples required for a split \texttt{min\_split}, the regularity factor \texttt{n\_reg}
		\item[\bf{Training}:]\\
		For $n=1:\texttt{ntree}$
		\begin{itemize}
			\item[1.] Draw \texttt{B} samples from $s_N$ with replacement to create $s_B^n$. Samples are drawn proportionally from each treatment.
			\item[2.] Build a tree from $s_B^n$. At each node, we draw \texttt{mtry} coordinates at random, then find the split with the largest increase in expected response among the \texttt{mtry} coordinates as measured by the splitting criterion defined in Eq.~(\ref{eqn:criterion}).
			\item[3.] The output of each tree is a partition of the feature space as represented by the terminal nodes, and for each terminal node, the estimation of the expected response under each treatment.
		\end{itemize}
		\item[{\bf Prediction}:] Given a new point in the feature space, the predicted expected response under a treatment is the average of the predictions from all the trees. The optimal treatment is the one with the largest predicted expected response.
	\end{algorithmic}
	\label{alg2}
\end{algorithm}

\section{Experimental Evaluation}\label{sec4}

In this section, we present an experimental comparison between the proposed CTS algorithm and other applicable uplift modeling methods on several benchmark datasets. The first dataset is generated from a 50-dimensional artificial data model. Knowing the true data model allows us to compare methods without worrying about model evaluation accuracy\footnote{Exact values of data model parameters and datasets can be found at this Dropbox link \url{https://www.dropbox.com/sh/sf7nu2uw8tcwreu/AAAhqQnaUpR5vCfxSsYsM4Tda?dl=0}}. Next, we compare the methods on two large-scale industry provided datasets. One is a single-treatment binary-response dataset on the purchase of priority boarding for flights. The other dataset is about the purchase of reserved seats on flights which has multi treatments and continuous response. On the latter two datasets we introduce the \emph{modified uplift curve} which is a convenient way of understanding the trade-off between the risk of exposing subjects to treatments and the gain from customizing treatment assignment. 

\subsection{Synthetic Data}\label{sec4-1}

The feature space is the fifty-dimensional hyper-cube of length 10, or $\mathbb{X}^{50}= [0, 10]^{50}$. Features are uniformly distributed in the feature space, i.e., $X_d \sim \mathrm{U}[\,0, 10\,]$, for $d=1,...,50$. There are four different treatments, $T=1,2,3,4$. The response under each treatment is defined as below. 
\begin{equation}
Y = \left\{
\begin{array}{rl}
f(X) + \mathrm{U}[0, \alpha X_1] + \epsilon  & \text{if } T=1, \\
f(X) + \mathrm{U}[0, \alpha X_2] + \epsilon & \text{if } T=2, \\
f(X) + \mathrm{U}[0, \alpha X_3] + \epsilon & \text{if } T=3, \\
f(X) + \mathrm{U}[0, \alpha X_4] + \epsilon & \text{if } T=4.
\end{array} \right.
\end{equation}
The response is the sum of three components. 
\begin{itemize}
	\item The first term $f(X)$ defines the systematic dependence of the response on the features and is identical for all treatments. Specifically, $f$ is chosen to be a mixture of $50$ exponential functions so that it is complex enough to reflect real-world scenarios. 
	\begin{align}
	& f(x_1,..., x_{50}) \\
	= &\sum_{i=1}^{50} a^i \cdot \exp \{ -b^i_1|x_1 - c^i_1|  - \cdots - b^i_{50}|x_{50} - c^i_{50}| \},\nonumber  
	\end{align}
	where $a^i$, $b^i_j$ and $c^i_j$ are chosen randomly. 

	\item The second term $\mathrm{U}[0, \alpha X_t]$ is the treatment effect and is unique for each treatment $t$. In many applications we would expect the treatment effect to be of a lower order of magnitude of the main effect, so we set $\alpha$ to be $0.4$ which is roughly $5\%$ of $\mathrm{E}[|f(X)|]$.
	
	\item  The third term $\epsilon$ is the zero-mean Gaussian noise, i.e. $\epsilon \sim \mathrm{N}(0, \sigma^2)$. Note that the standard deviation $\sigma$ of the noise term is identical for all treatment. $\sigma$ is set to $0.8$ which is twice the amplitude of the treatment effect $\alpha$.
	 
\end{itemize}
Under this particular data model, the expected response is the same for all treatments, i.e., $\mathrm{E}[Y|T=t] = 5.18$ for $t=1,2,3,4$. The expected response under the optimal treatment rule $\mathrm{E}[Y|T=h^*(X)]$ is $5.79$.

We compare the performance of 5 different methods that are applicable to multi-treatment uplift problems with continuous response. They are CTS, Separate Model Approach with Random Forest (SMA-RF), K-Nearest Neighbor (SMA-KNN), Support Vector Regressor with Radial Basis Kernel (SMA-SVR), and AdaBoost (SMA-Ada). CTS is implemented in \emph{R} by the authors. For other algorithms, we use the implementation in \emph{scikit-learn}, a popular machine learning library in Python. These algorithms are tested under increasing training data size, specifically 500, 2000, 4000, 8000, 16000, and 32000 samples per treatment. For each size, 10 training data sets are generated so that we can compute the margin of error of the results. The performance of a model is evaluated using Monte Carlo simulation and the true data model. All models are tuned carefully with validation or cross-validation. Detail of the parameter selection procedure specific to each algorithm can be found in the Appendix.

\begin{figure}[hbt]
	\centering
	\includegraphics[width=\linewidth]{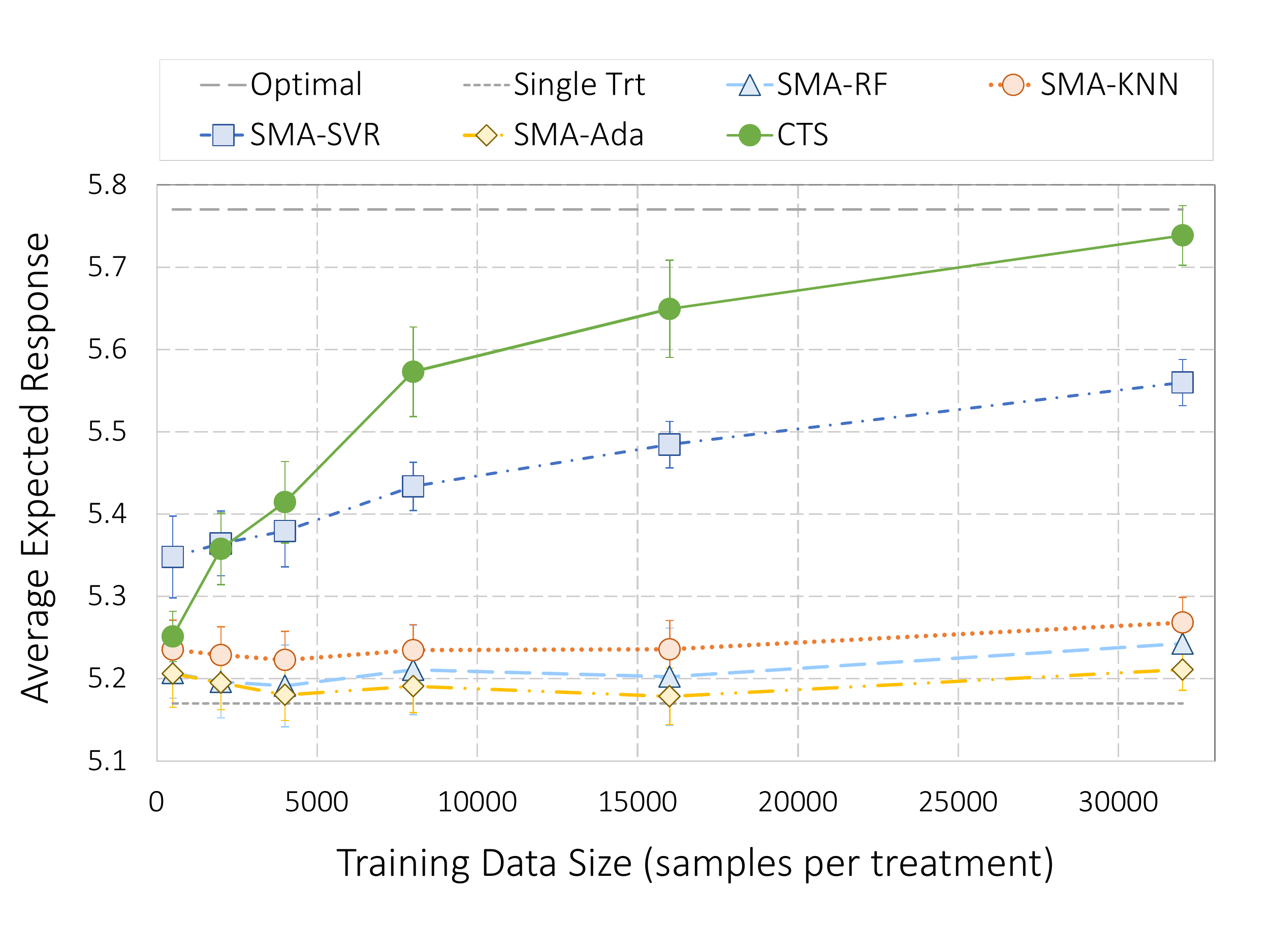}
	\caption{Averaged expected response of different algorithms on the synthetic data. The 95\% margin of error is computed with results from 10 different training datasets. For each data size, all algorithms are tested on the same 10 datasets.  }
	\label{fig:simulated_data}
\end{figure}

The performance of the 5 methods are plotted in Fig.~\ref{fig:simulated_data}. For reference, we also plot the single treatment expected response (short dash line without markers) and the optimal expected response (long dash line without markers). The vertical bars are the 95\% margin of error computed with results from 10 training datasets. To ensure consistency in comparison, for each data size, all the methods are tested with the same 10 datasets. From the figure we can see that CTS surpasses the separate model approach when data size is 2000, and the advantage continues to grow as the training size increases. At 32000 sample per treatment, the performance of CTS is very close to the oracle performance. Among the algorithms for the separate model approach, support vector regressor with radial basis kernel performs the best. This is not surprising considering the true data model is basically a mixture of exponentials. If the model for each treatment is accurate enough, the separate model approach can also create uplift. It is worth emphasizing the performance difference between CTS and Random Forest. The only essential difference between the two algorithms is the splitting criterion, and yet their performance is far from similar. Even with the largest training size, SMA-RF (dash line with triangle markers) does only slightly better than assigning a fixed treatment. This example again, shows the importance of developing specialized algorithms for uplift modeling.

\subsection{Priority Boarding Data}\label{sec4-2}

One application for randomized experiment is the online pricing problem where a treatment is a candidate price of the product. Customers are randomly assigned to different prices and the one that leads to the highest profit per customer is selected. A major European airline applied this method to select the price for the priority boarding of flights. The control is the default price \EUR{5} and the treatment is \EUR{7} . Interestingly, these two prices lead to the same revenue per passenger overall - \EUR{0.42}. 

With the help of our industry collaborators, we are able to access the randomized experiment data. After initial analysis, we confirm that the purchasing behavior of passengers varies significantly and it can be beneficial to customize price offering based on certain characteristics. A total of 9 features are derived based on the information of flight and of the reservation. These are the origin station, the origin-destination pair, the departure weekday, the arrival weekday, the number of days between flight booking and departure, flight fare, flight fare per passenger, flight fare per passenger per mile, and the group size. 

The data is randomly split into the training set (225,000 samples per treatment) and the test set (75,000 samples per treatment). Six methods are tested. They are the separate model approach with Random Forest (SMA-RF), Support Vector Machine (SMA-SVM), Adaboost (SMA-Ada), K-Nearest Neighbors (SMA-KNN), as well as the uplift Random Forest method implemented in \cite{Ruplift}, and CTS. For the first 5 methods, customer decision is modeled as binary response, 1 for purchase and 0 for non-purchase. Expected revenue is then calculated as the product of the purchase probability and the corresponding price. With CTS, we directly model the revenue as the (discrete) response. All the algorithms are carefully tuned with cross-validation. During cross-validation, the performance of a model is estimated on the hold-out set as measured by Eq.~(\ref{eqn:zbar}). See Appendix for detail on parameter tuning.

In many applications, exposing subjects to treatments involves a certain level of risk, such as disruption to customer experience, unexpected side effects, etc. As a result, we may want to limit the percentage of population exposed to treatment while still obtaining as much benefit from customization as possible. To measure the performance of an uplift model in this respect, we introduce the \textbf{modified uplift curve}, in which the horizontal axis is the percentage of population subject to treatments and the vertical axis is the expected response. Given an uplift model, we can compute, for each test subject, the difference in expected response under the predicted optimal treatment and the control. Then we rank the test subjects by the difference from high to low. For a given percentage $p$, we assign the top $p$ percent of the test subjects to their corresponding optimal treatment as predicted by the model, and the rest to the control. The expected response under this assignment is then estimated with Eq.~(\ref{eqn:zbar}). 

Fig.~\ref{fig:muc_pb} shows the modified uplift curves for the 6 methods under comparison on the priority boarding data. CTS performs the best at all population percentage. The upliftRF algorithm ranks the second and outperforms the separate model approach. The SMA-RF is very accurate in terms of identifying subpopulation for which the treatment is highly beneficial (the sharp rise at the beginning of the curve) or extremely harmful (the sharp decline at the end). Yet it fails to predict the treatment effect for the vast majority which is demonstrated by the (almost) straight line for the middle part. SMA-SVM and SMA-KNN perform poorly on this data set which we think partly due to their limitations in handling categorical variables.

\begin{figure}[hbt]
\centering
\includegraphics[width=\linewidth]{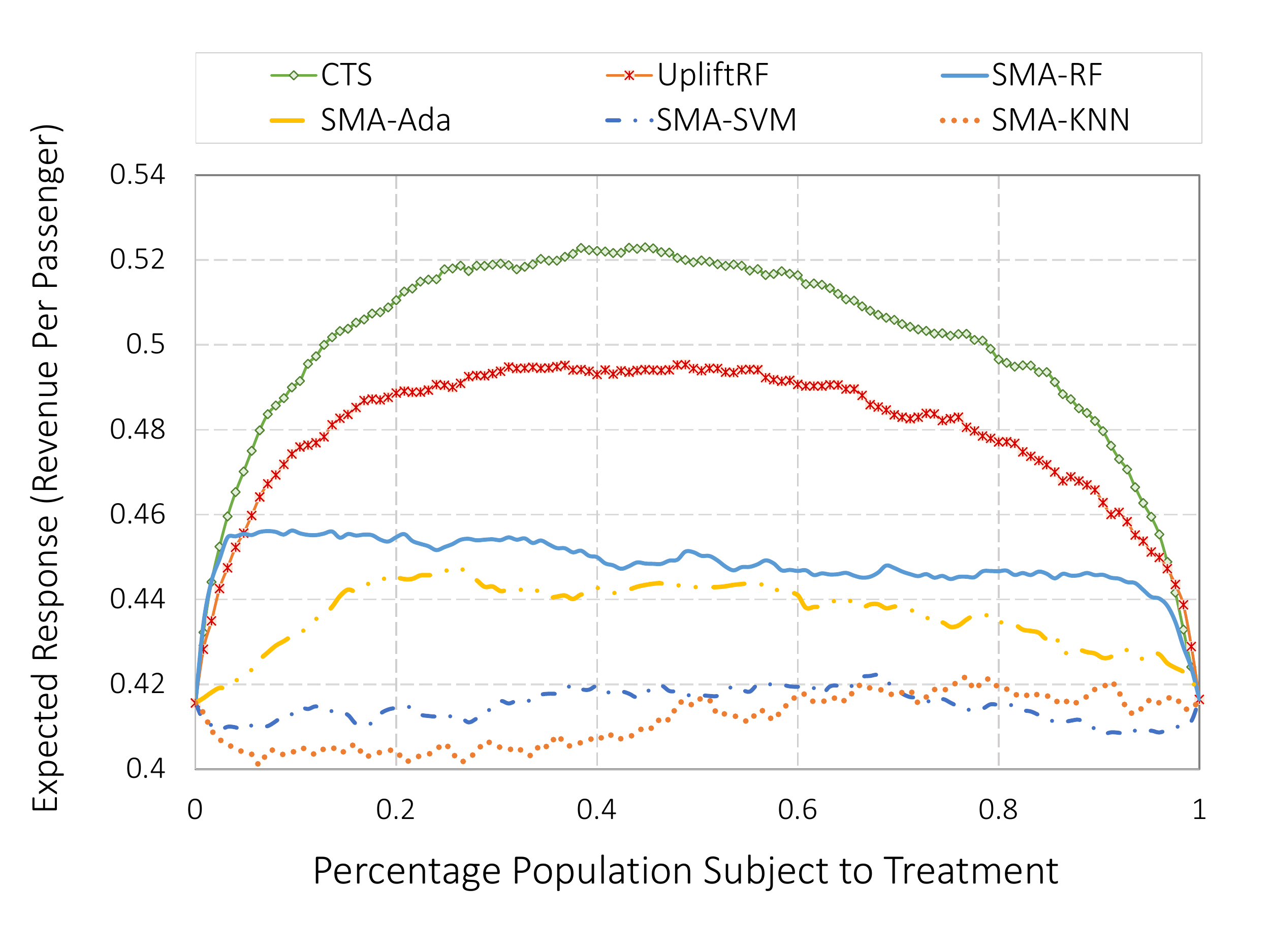}
\caption{Modified uplift curves of different algorithms for the priority boarding data.}
\label{fig:muc_pb}
\end{figure}

\begin{figure}[hbt]
\centering
\includegraphics[width=\linewidth]{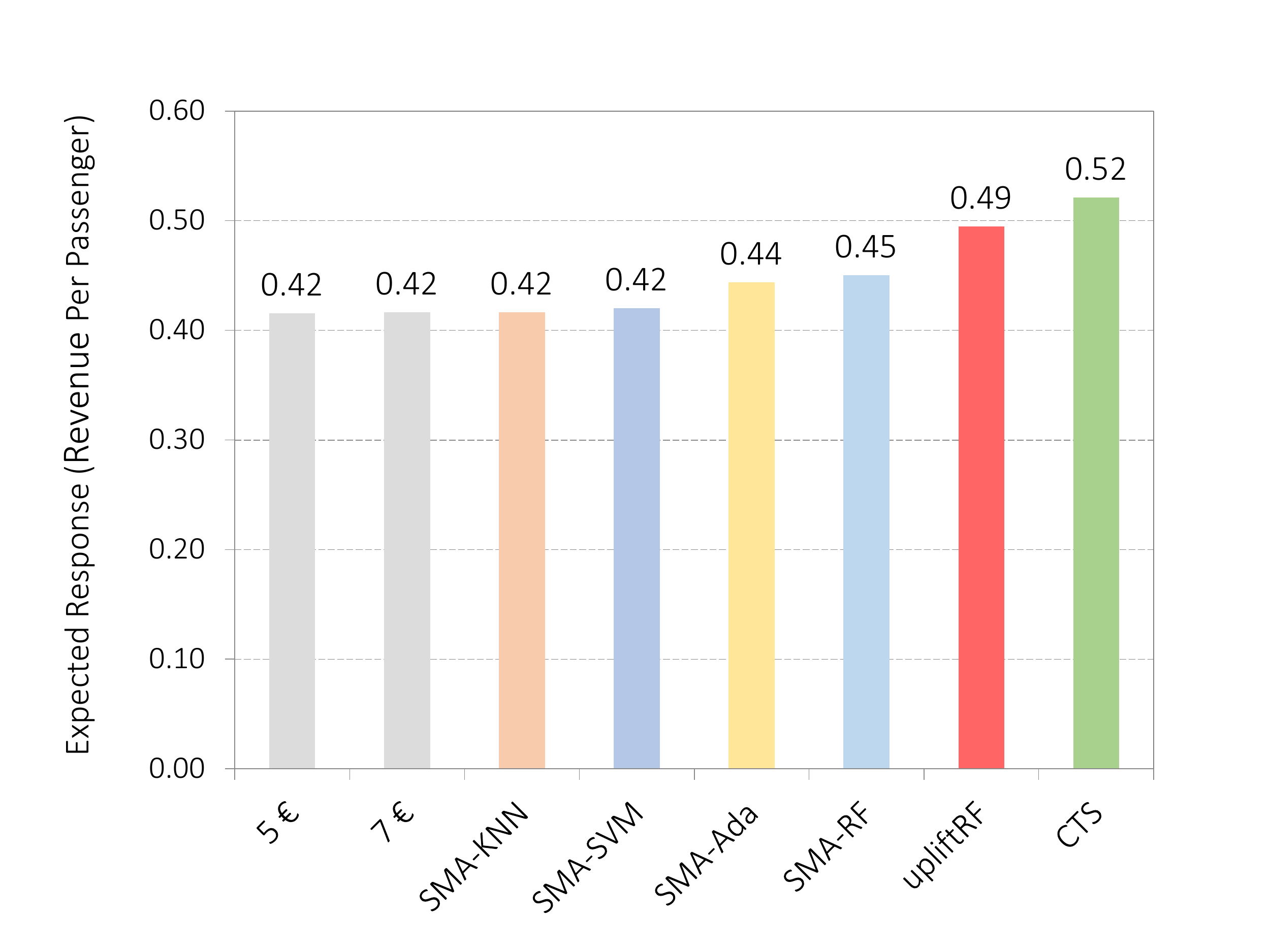}
\caption{Expected revenue per passenger from priority boarding based on different models.}
\label{fig:bar_pb}
\end{figure}

Fig.~\ref{fig:bar_pb} plots the expected revenue per passenger  if all the test subjects are assigned the predicted optimal treatment. As can be seen from the figure, customized pricing models can significantly increase the revenue. The increase in the revenue per passenger from \euro0.42 to \euro0.52 could lead to a remarkable gain in profit for an airline with tens of millions of scheduled passengers per year. This test case demonstrates the benefit of designing and applying specialized algorithms for uplift modeling.

\subsection{Seat Reservation Data}

Another airline experiments with seat reservation price. Treatments are price levels -  low (L), medium low (ML), medium high (MH) and high (H). The response is the revenue from each transaction. Because the same price level may correspond to different prices on different routes and one transaction may include the purchase of multiple seats, we need to model the response as a continuous variable. The features we use include the booking hour, the booking weekday, the travel hour, the travel weekday, the number of days between the last login and the next flight, the fare class, the zone code (all flight routes are divided into 3 zones, and prices are set for different zones), whether the passenger returns to the website after flight ticket purchase, the journey travel time, the segment travel time, the number of passengers, and the quantity available.

The number of samples for the four treatments are 213,488, 176,637, 160,576, 214,515. We use 50\% for training, 30\% for validation, and 20\% for test.  We compare the performance of CTS and SMA-RF in this test. We choose SMA-RF because it is the best among the Separate Model Approach on the priority boarding data.  UpliftRF is not included because it can not be applied to continuous response models. 

\begin{figure}[hbt]
	\centering
	\includegraphics[width=\linewidth]{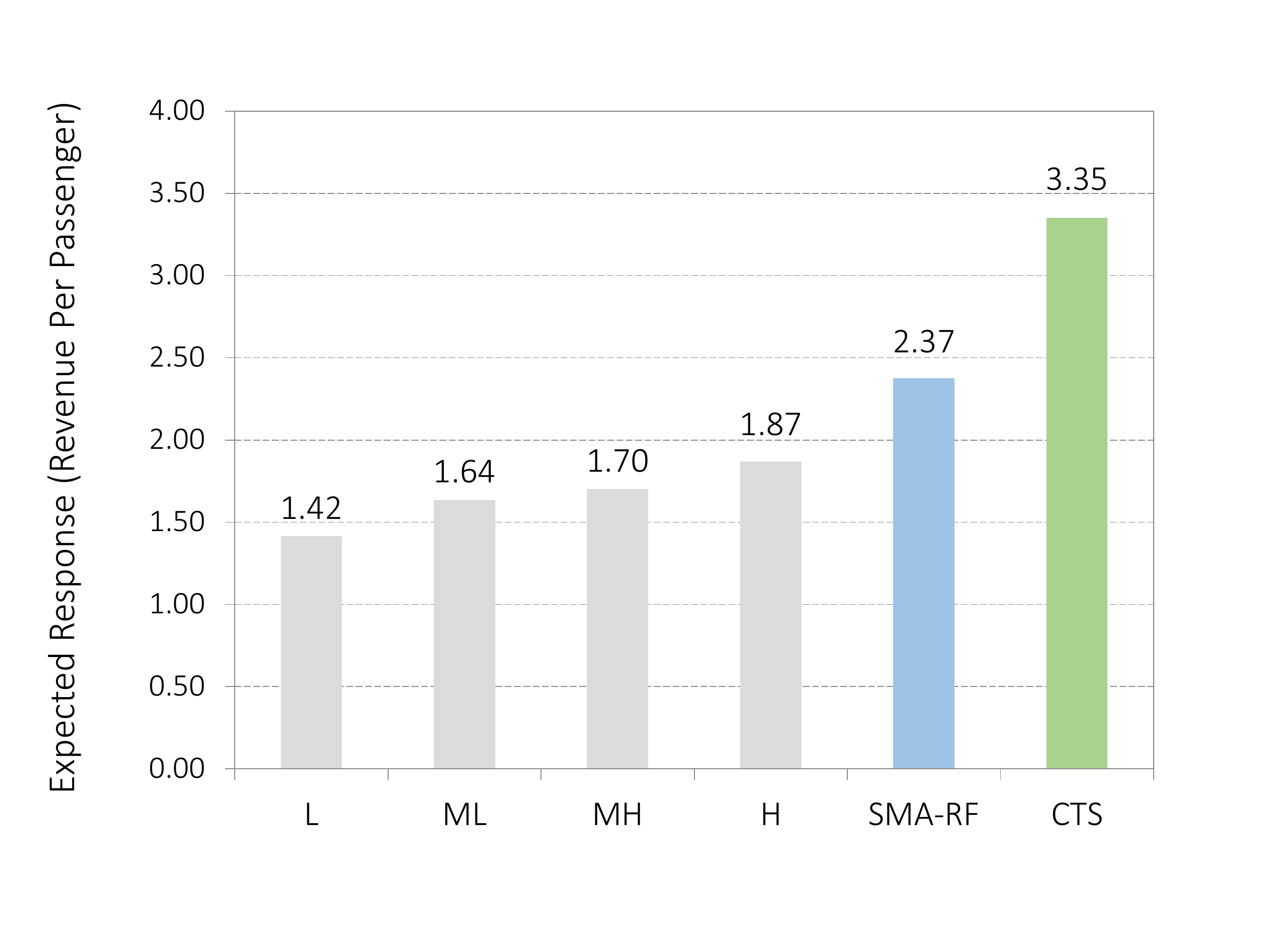}	\caption{Expected revenue per passenger from seat reservation when applying different pricing models.}
	\label{fig:bar_seat}
\end{figure}

\begin{figure}[hbt]
\centering
\includegraphics[width=\linewidth]{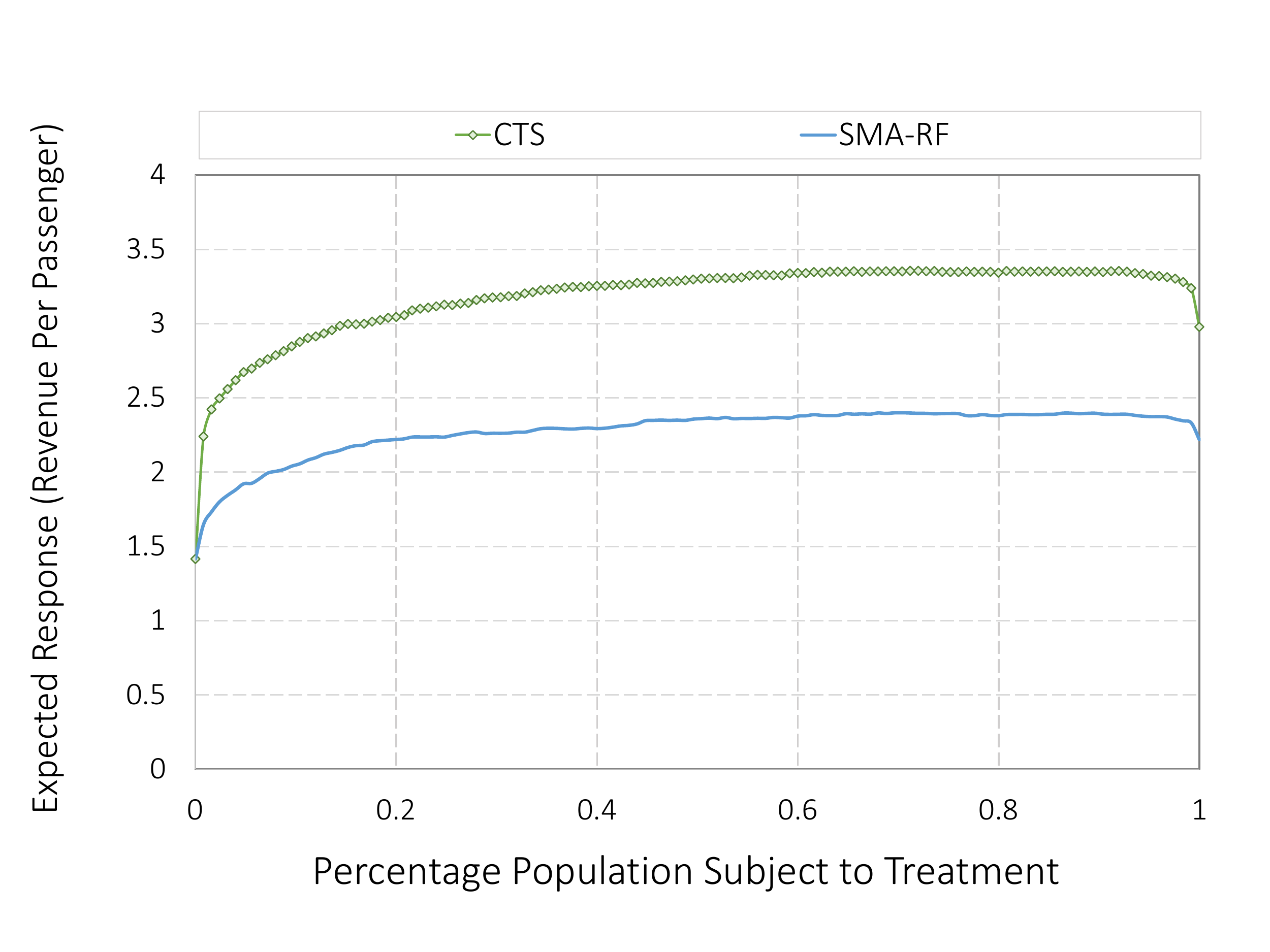}
	\caption{Modified uplift curves of SMA-RF and CTS on the seat reservation data.}
	\label{fig:muc_seat}
\end{figure}

The average revenue per customer with different pricing models is shown in Fig.~\ref{fig:bar_seat}.
The optimal single price level is H with an expected revenue of {\$}1.87 per passenger. By personalizing treatment assignment, we can achieve \$2.37 with SMA-RF and \$3.35 with CTS. Fig.~\ref{fig:muc_seat} shows the modified uplift curves of SMA-RF and CTS. We can see that CTS outperforms SMA-RF at every population percentage. By employing a specialized algorithm for uplift modeling the airline can significantly improve its profit margin.

\section{Conclusion}\label{sec5}

Uplift modeling initially gathered attention with its successful application in marketing and insurance. But it does not need to be restricted to these domains. Any situation where personalized treatment selection is desired and randomized experiment is possible can be a potential use case for uplift modeling. As an example, we described in Section~\ref{sec4} how to apply it to customized pricing. 

The contribution of this paper to Uplift Modeling is threefold. First, we present a way to obtain an unbiased estimate of the expected response under an uplift model which has not been available in the literature. Second, we design a tree ensemble algorithm with a splitting criterion based on the new estimation method. Both the unbiased estimate and the algorithm apply naturally to multiple treatments and continuous response, which significantly extends the current focus of uplift algorithms on single-treatment binary-response cases. Lastly, we showed that our algorithm lead to 15\% - 40\% more revenue than non-uplift algorithms with the priority boarding and seat reservation data, which demonstrated the impact of uplift modeling on  customized pricing.

\section*{Acknowledgment}
This work was supported in part by Accenture through the Accenture-MIT Alliance in Business Analytics.

%
%
%
%


\section*{Appendix}\label{appendix}

\subsubsection*{Synthetic Data} Here we describe the details of parameter tuning in Section~\ref{sec4}.

\textbf{CTS}: Fixed parameters are \texttt{ntree=100}, \texttt{mtry=25} and \texttt{n\_reg=3}. The value of \texttt{min\_split} is selected among \texttt{[25, 50, 100, 200, 400, 800, 1600, 3200, 6400]} (Large values are omitted when they exceeds dataset size). \texttt{min\_split} is selected by 5-fold cross-validation when training data size below 8000 sample-per-treatment, otherwise by validation (half training/half test) on one data set and kept the same for other nine data sets. 
	
\textbf{RF}: Fixed parameters are \texttt{n\_estimators=100} and \texttt{max\_features=25}. Parameter \texttt{nodesize} is tuned with 5-fold cross-validation among \texttt{[1,5,10,20]}.
	
\textbf{KNN}: Parameter \texttt{n\_neighbors} is tuned with 5-fold cross-validation among \texttt{[5,10,20,40]}.
	
\textbf{SVR}: The regularization parameter C and the value of the insensitive-zone $\epsilon$ are determined analytically using the method proposed in \cite{Cherkassky2004}. The spread parameter of the radial basis kernel $\gamma$ is selected among $[10^{-4},10^{-3},10^{-2},10^{-1}]$ using 5-fold cross-validation.
	
\textbf{Ada}: Square loss with \texttt{n\_estimators=100}.

\subsubsection*{Priority Boarding/Seat Reservation Data} Models are tuned similarly as with synthetic data except the following. 1. For CTS, upliftRF, and SMA-RF, \texttt{mtry=3}. 2. For CTS, \texttt{min\_split} is selected among \texttt{[25, 50, 100, 200]}. For upliftRF, \texttt{min\_split} is selected among \texttt{[5, 10, 20, 40]}. Parameter selection is conducted by validation because of the time constraint.

\end{document}